\documentclass[11pt]{article}
\usepackage{listings}
\usepackage{pdflscape}
\usepackage{amsfonts}
\usepackage{epsfig}
\usepackage{lscape}
\usepackage{longtable}
\usepackage{amsmath,amssymb,amsthm}
\usepackage{graphicx}
\usepackage{subfig}
\usepackage{color}
\usepackage{placeins}
\usepackage{url}
\usepackage{cases}
\usepackage{longtable}
\usepackage{supertabular}
\usepackage{multirow}

\usepackage{amsmath}
\allowdisplaybreaks[4]
\usepackage{cite}
\usepackage{algorithmic}
\usepackage{algorithm,float}
\usepackage{booktabs}
\usepackage{enumitem}
\usepackage{array}
\usepackage{lipsum}
\usepackage{hyperref}
\usepackage{cleveref}
\usepackage{svg}
\usepackage{xcolor}



\newtheorem{theorem}{Theorem}

\newtheorem{lemma}{Lemma}

\newtheorem{corollary}{Corollary}

\usepackage{caption}
\begin{document}

\title{\bf Accelerating RLHF Training with Reward Variance Increase \footnotemark[1]}
\author{Zonglin Yang\footnotemark[2], \quad Zhexuan Gu\footnotemark[3], \quad Houduo Qi\footnotemark[4], \quad Yancheng Yuan\footnotemark[5]}
\date{\today}
\maketitle

\renewcommand{\thefootnote}{\fnsymbol{footnote}}
\footnotetext[1]{{\bf Funding:} The work of Houduo Qi was supported by the Hong Kong Research Grants Council (Project No.
15303124). The work of Yancheng Yuan was supported by the Hong Kong Research Grants Council (Project No.
25305424) and the Research Center for Intelligent Operations Research at The Hong Kong Polytechnic University.}
\footnotetext[2]{Department of Applied Mathematics, The Hong Kong Polytechnic University, Hung Hom, Hong Kong ({\tt zonglin.yang@connect.polyu.hk}).}
\footnotetext[3]{Department of Applied Mathematics, The Hong Kong Polytechnic University, Hung Hom, Hong Kong ({\tt zhexuan.gu@connect.polyu.hk}).}
\footnotetext[4]{Department of Data Science and Artificial Intelligence and Department of Applied Mathematics, The Hong Kong Polytechnic University, Hung Hom, Hong Kong ({\tt houduo.qi@polyu.edu.hk}).}
\footnotetext[5]{Department of Applied Mathematics, The Hong Kong Polytechnic University, Hung Hom, Hong Kong (\textbf{Corresponding author}. {\tt yancheng.yuan@polyu.edu.hk}).}
\renewcommand{\thefootnote}{\arabic{footnote}}

\begin{abstract}
Reinforcement learning from human feedback (RLHF) is an essential technique for ensuring that large language models (LLMs) are aligned with human values and preferences during the post-training phase. As an effective RLHF approach, group relative policy optimization (GRPO) has demonstrated success in many LLM-based applications. However, efficient GRPO-based RLHF training remains a challenge. Recent studies reveal that a higher reward variance of the initial policy model leads to faster RLHF training. Inspired by this finding, we propose a practical reward adjustment model to accelerate RLHF training by provably increasing the reward variance and preserving the relative preferences and reward expectation. Our reward adjustment method inherently poses a nonconvex optimization problem, which is NP-hard to solve in general. To overcome the computational challenges, we design a novel $O(n \log n)$ algorithm to find a global solution of the nonconvex reward adjustment model by explicitly characterizing the extreme points of the feasible set. As an important application, we naturally integrate this reward adjustment model into the GRPO algorithm, leading to a more efficient GRPO with reward variance increase (GRPOVI) algorithm for RLHF training. As an interesting byproduct, we provide an indirect explanation for the empirical effectiveness of GRPO with rule-based reward for RLHF training, as demonstrated in DeepSeek-R1. Experiment results demonstrate that the GRPOVI algorithm can significantly improve the RLHF training efficiency compared to the original GRPO algorithm.
\end{abstract}

\newcommand{\keywords}[1]{%
  \par\vspace{0.5\baselineskip}%
  \noindent{\textbf{Keywords:} #1}
}

\keywords{Large language models, reinforcement learning from human feedback, group relative policy optimization, nonconvex optimization}

\section{Introduction}
In recent years, large language models (LLMs) have achieved unprecedented advancements, leading to the emergence of notable LLM-based products, such as ChatGPT \cite{achiam2023gpt}, Gemini \cite{team2023gemini} and DeepSeek \cite{liu2024deepseek, guo2025deepseek}. Broadly speaking, LLM training consists of two primary phases: pre-training and post-training. The pre-training phase aims to equip language models with general text comprehension and generation capabilities through exposure to large-scale corpora \cite{radford2018improving, brown2020language}. However, despite their potential, pre-trained LLMs may still produce biased, toxic, or otherwise harmful content \cite{bommasani2021opportunities, zheng2023secrets, dai2024bias}. To address these issues, post-training becomes essential, as it further refines the models' abilities for task-specific adaptation \cite{touvron2023llama} and aligns their outputs with human values \cite{christiano2017deep, ziegler2019fine}. 

With the success of ChatGPT and related products, reinforcement learning from human feedback (RLHF) has proven to be an effective post-training approach for aligning LLMs with human values and preferences \cite{ouyang2022training}. Beyond alignment, RLHF also enables models to develop Chain-of-Thought (CoT) reasoning capabilities for addressing complex problems \cite{shao2024deepseekmath}. Typically, the standard RLHF framework comprises three key stages \cite{ziegler2019fine, bai2022training, rafailov2023direct, ji2025overview}: supervised fine-tuning (SFT), reward model training, and reward maximization (via the policy gradient algorithm). During the SFT stage, a pre-trained LLM is fine-tuned over a vast of supervised tasks \cite{razin2025makes, zhao2023survey, ouyang2022training}. After the SFT stage, the LLM needs to be further aligned with human values and preferences, which is critical for model safety, helpfulness, truthfulness, and so on \cite{chujie2024honestllm, liu2024large, tuan2024towards, wu2025navigating}. To achieve this objective, a reward model is trained on a labeled dataset of human feedback with human preferences \cite{busa2014preference, christiano2017deep}. Finally, RLHF employs reinforcement learning methods to fine-tune the LLM for maximizing the reward expectation (usually by the policy gradient algorithm) \cite{zhao2023survey, xu2020sample}. 

Among different reinforcement learning training algorithms in RLHF, the proximal policy optimization (PPO) \cite{schulman2017proximal} is widely adopted in LLM fine-tuning. In this framework, the language model (policy model) is optimized to maximize the expected reward for its generated responses, while being constrained to limit excessive divergence from the reference model. The advantage function in PPO is estimated through generalized advantage estimation (GAE) \cite{schulman2015high}, which requires training an additional value model. This multi-model architecture substantially increases computational costs during training. To mitigate this issue, group relative policy optimization (GRPO) has been proposed, eliminating the need for the additional value model in PPO \cite{shao2024deepseekmath}. GRPO estimates advantages by normalizing rewards across responses generated for the same prompt. This method has been successfully implemented in DeepSeek-R1 \cite{guo2025deepseek}, drawing wide attention in the AI community \cite{liu2025understanding, yang2025qwen3, yu2025dapo}. However, GRPO-based RLHF still suffers from its slow training problem \cite{li2025disco, zhang2025grpo}.

Very recently, Razin et al \cite{razin2025makes} have theoretically demonstrated that the efficiency of RLHF training is highly correlated with the reward variance of the initial policy model (over the response space). Specifically, a higher reward variance leads to faster RLHF training. This directly implies that increasing the reward variance of the initial policy model can contribute to accelerating the GRPO-based RLHF training. Following this finding, Xu et al \cite{xu2025not} apply a down-sampling technique to select a subset of responses for the same prompt with higher reward variance. However, this approach needs to generate many more responses per prompt to find such subsets, requiring much more computational time (generating long responses from an LLM is slow). Indeed, it does not change the reward variance over the whole response space. To the best of our knowledge, a practical algorithm that benefits from the theoretical findings in \cite{razin2025makes} (in other words, increasing the reward variance over the response space) to accelerate (GRPO-based) RLHF training is still absent. 

In this paper, we propose a novel reward adjustment model to provably increase reward variance over the response space, while preserving the reward expectation and preference of the generated responses for the prompt. The proposed model is nonconvex, which is NP-hard to solve in general. We carefully investigate the structure of the extreme points of the constraints in our proposed model, which surprisingly leads to an efficient $O(n\log n)$ algorithm to find a global solution. Consequently, we can naturally integrate this reward adjustment method into the GRPO algorithm, leading to a more efficient GRPO with reward variance increase (GRPOVI) algorithm for RLHF training. Experiment results demonstrate that the GRPOVI algorithm can significantly improve the training
efficiency compared to the original GRPO algorithm. We summarize the key contributions of this paper as follows:
\begin{enumerate}
    \item We propose a novel reward adjustment model to provably increase reward variance over the response space, while preserving the reward expectation and preference of the generated responses for the prompt.
    \item We designed an efficient $O(n\log n)$ algorithm to find a global solution to the nonconvex reward adjustment model by explicitly characterizing the extreme points of the feasible set. As an interesting byproduct, we provide an indirect explanation for the empirical effectiveness of GRPO with rule-based reward for RLHF training as demonstrated in DeepSeek-R1.
    \item We proposed a novel GRPOVI algorithm with marginal per-iteration cost compared to the original GRPO algorithm to accelerate the RLHF training.
    \item Numerical experiment results demonstrate the superior performance of GRPOVI compared to the original GRPO in RLHF training. 
\end{enumerate}

The rest of this paper is organized as follows. Some necessary preliminaries are introduced in \cref{sec:preliminaries}.
We introduce the reward adjustment model with some theoretical guarantees in \cref{sec:rewardadjust}.  An efficient $O(n\log n)$ algorithm to find a global solution to the nonconvex reward adjustment model is also presented in this section. We introduce the GRPOVI algorithm for efficient RLHF training in \cref{sec:GRPOalg}.  Experiment results are shown in \cref{sec:experiments}. We conclude the paper in \cref{sec:conclusions}.

\section{Preliminaries}
\label{sec:preliminaries}

Let $\mathbb{V}$ be a finite vocabulary of tokens, and $\Theta_1, \Theta_2$ be finite-dimensional Euclidean spaces. Given a prompt $\mathbf{x} \in \mathcal{X} \subseteq \mathbb{V}^L$ ($L$ is the maximum number of tokens), the language model $\pi_\theta$ parameterized by $\theta \in \Theta_1$ generates a probability distribution over the response space $\mathcal{Y} \subseteq \mathbb{V}^L$. The likelihood of any response $\mathbf{y} \in \mathcal{Y}$ is given by
\begin{equation}
\label{eq: autoregress}
\pi_\theta(\mathbf{y}|\mathbf{x}) = \prod _{t=1}^{|\mathbf{y}|} \pi_\theta (\mathbf{y}_t|\mathbf{x}, \mathbf{y}_{<t}),
\end{equation}
where $\mathbf{y}_t$ and $\mathbf{y}_{<t}$ denote the $t$-th token and preceding tokens, respectively, and $|\mathbf{y}|$ represents the number of tokens in the response $\mathbf{y}$.

Let $\pi_{\mathrm{ref}}$ be the obtained reference model after SFT, $r$ be the reward model. The language model (policy model) $\pi_\theta$ is aligned to human values and preferences by using policy gradient methods to maximize
\begin{equation}
\label{eq: RLHFobjective}
    \mathcal{L} _ \text{RLHF} (\theta) = \mathbb{E} _ {\mathbf{x}\sim \mathcal{D}} \left[ \mathbb{E} _ {\mathbf{y}\sim \pi_\theta (\cdot|\mathbf{x})} [r(\mathbf{x}, \mathbf{y})] - \lambda D_{\text{KL}}(\pi_\theta(\cdot|\mathbf{x}) || \pi_{\text{ref}}(\cdot|\mathbf{x})) \right],
\end{equation}
where $\mathcal{D} \subset \mathcal{X}$ is a training prompt set \cite{razin2025makes}, and $D_{\mathrm{KL}}(\cdot||\cdot)$ is the KL divergence. Currently, the most popular methods for RLHF are PPO \cite{schulman2017proximal} and GRPO \cite{shao2024deepseekmath}. The objective function in PPO is 
\begin{equation*}
    \mathcal{L}_{\mathrm{PPO}}(\theta) = \mathbb{E}_{\mathbf{x} \sim \mathcal{D}, \mathbf{y} \sim \pi_{\theta_{\mathrm{old}}}(\cdot | \mathbf{x})} \frac{1}{|\mathbf{y}|} \sum_{t=1}^{|\mathbf{y}|} \min \left\{ \frac{\pi_\theta (\mathbf{y}_t | \mathbf{x}, \mathbf{y}_{<t})}{\pi_{\theta_{\mathrm{old}}} (\mathbf{y}_t | \mathbf{x}, \mathbf{y}_{<t})}A_t, \mathrm{clip} \left( \frac{\pi_\theta (\mathbf{y}_t | \mathbf{x}, \mathbf{y}_{<t})}{\pi_{\theta_{\mathrm{old}}} (\mathbf{y}_t | \mathbf{x}, \mathbf{y}_{<t})} , \epsilon \right) A_t \right\} ,
\end{equation*}
where $\pi_\theta$ and $\pi_{\theta_{\mathrm{old}}}$ represent the policy model after update and before update, respectively, $\mathrm{clip}(\cdot, \epsilon) = \min \{ 1+\epsilon, \max \{1-\epsilon,\cdot\} \}$. The advantage $A_t$ is estimated by GAE \cite{schulman2015high}, which is computed based on the reward $r(\mathbf{x},\mathbf{y})$, the likelihoods $\pi_\theta (\mathbf{y}_t | \mathbf{x}, \mathbf{y}_{<t}), \pi_{\theta_{\mathrm{old}}} (\mathbf{y}_t | \mathbf{x}, \mathbf{y}_{<t})$, and a learned value function $V_\phi$, where $\phi \in \Theta_2$ represent the parameters of the value model. The PPO method trains both the policy model and the value model, requiring high memory and computation resources \cite{shao2024deepseekmath}.

To mitigate this issue, GRPO has been proposed to train the policy model without the value model \cite{shao2024deepseekmath, guo2025deepseek}. In GRPO, the advantage is estimated based on the rewards within a response group for the same prompt. The objective of GRPO is
\begin{equation*}
\begin{aligned}
    & \mathcal{L}_{\mathrm{GRPO}}(\theta) 
    = \mathbb{E}_{\mathbf{x} \sim \mathcal{D}, \{\mathbf{y}_i\}_{i=1}^n \sim \pi_{\theta_{\mathrm{old}}}(\cdot | \mathbf{x})} \frac{1}{n} \sum_{i=1}^n \frac{1}{|\mathbf{y}_i|} \sum_{t=1}^{|\mathbf{y}_i|} \\
    & \left( \min \left\{ \frac{\pi_\theta (\mathbf{y}_{i,t} | \mathbf{x}, \mathbf{y}_{i,<t})}{\pi_{\theta_{\mathrm{old}}} (\mathbf{y}_{i,t} | \mathbf{x}, \mathbf{y}_{i,<t})}\hat{A}_{i,t},
    \mathrm{clip} \left( \frac{\pi_\theta (\mathbf{y}_{i,t} | \mathbf{x}, \mathbf{y}_{i,<t})}{\pi_{\theta_{\mathrm{old}}} (\mathbf{y}_{i,t} | \mathbf{x}, \mathbf{y}_{i,<t})} , \epsilon \right) \hat{A}_{i,t} \right\} -\lambda D_{\mathrm{KL}}(\pi_\theta(\cdot|\mathbf{x}) || \pi_{\text{ref}}(\cdot|\mathbf{x})) \right),
\end{aligned}
\end{equation*}
where the response group $\{\mathbf{y}_i\}_{i=1}^n$ is generated for the same prompt $\mathbf{x}$, and the advantage $\hat{A}_{i,t}$ is estimated by normalizing the rewards within the response group
\begin{equation*}
    \hat{A}_{i,t} = \frac{r_i - \mathrm{mean}\{ r(\mathbf{x},\mathbf{y}_1), r(\mathbf{x},\mathbf{y}_2), \cdots, r(\mathbf{x},\mathbf{y}_n) \}}{\mathrm{std}\{ r(\mathbf{x},\mathbf{y}_1), r(\mathbf{x},\mathbf{y}_2), \cdots, r(\mathbf{x},\mathbf{y}_n) \} }.
\end{equation*}
Compared to PPO, GRPO has exhibited comparable performance with less computational cost \cite{shao2024deepseekmath, guo2025deepseek}. However, GRPO still suffers from its slow training in practice \cite{li2025disco, zhang2025grpo}.

Very recently, Razin et al \cite{razin2025makes} theoretically demonstrated that the efficiency in optimizing \eqref{eq: RLHFobjective} based on \eqref{eq: autoregress} is highly correlated to the reward variance of the initial policy model over the response space. Specifically, let $\pi_{\theta(t)}$ be the policy model during the training process, where $\theta(t) \in \Theta_1$ represents the value of the model parameters $\theta$ at time $t$. For each prompt $\mathbf{x} \in \mathcal{X}$ and a target improvement $\gamma > 0$, it follows from \cite[Theorem 4]{razin2025makes} that the minimal time required to achieve a reward expectation increase $\gamma$ is 
\begin{equation*}
    \Omega\left(\mathbb{E}_{\mathbf{x}' \sim \mathcal{D}} \left[ \mathrm{Var} _ {\mathbf{y} \sim \pi_{\theta(0)} (\cdot|\mathbf{x}') } [r(\mathbf{x}',\mathbf{y})] \right]^{-1/3}\right),
\end{equation*}
where the reward variance is 
\[
    \mathrm{Var} _ {\mathbf{y} \sim \pi_{\theta} (\cdot|\mathbf{x}) } [r(\mathbf{x},\mathbf{y})] = \mathbb{E}_{\mathbf{y} \sim \pi_\theta(\cdot|\mathbf{x})} \left[ (r(\mathbf{x} , \mathbf{y}) - \mathbb{E}_{\mathbf{y}' \sim \pi_\theta(\cdot | \mathbf{x})}[r(\mathbf{x} , \mathbf{y}')])^2\right].
\]
Notably, this theoretical result applies to any general bounded reward model $r$. It explicitly shows that higher reward variance of the initial policy model over the response space leads to faster RLHF training. In their paper, Razin et al \cite{razin2025makes} have also experimentally verified that a lower reward variance leads to slower GRPO-based RLHF training. However, for a given reward model at hand, how to benefit from the theoretical findings to accelerate the RLHF training is still unknown.

\section{A Reward Adjustment Model}
\label{sec:rewardadjust}
In this section, we will introduce a novel reward adjustment model that can be naturally incorporated into the GRPO algorithm to accelerate the RLHF training for LLMs. 

\subsection{Optimization Formulation}
\label{subsec:optform}

Assume the original reward model is $r:\mathcal{X}\times \mathcal{Y}\to [m,M]$ for some $-\infty < m < M < \infty$. For any prompt $\mathbf{x} \in \mathcal{D}$ in the RLHF training dataset,  let $\mathbf{y}_1,\mathbf{y}_2\cdots,\mathbf{y}_n \in \mathcal{Y}$ be a group of responses for the prompt $\mathbf{x} \in \mathcal{D}$. Our goal is to adjust the original rewards $\{r(\mathbf{x}, \mathbf{y}_i)\}_{i=1}^n$ to $\{\tilde{r}(\mathbf{x}, \mathbf{y}_i)\}_{i=1}^n$ to increase reward variance, while preserving the reward expectation and the relative preferences between $\mathbf{y}_i$ and $\mathbf{y}_j$ for all $i$ and $j$.  

For simplicity, we denote $r_i = r(\mathbf{x},\mathbf{y}_i)$ for $1 \leq i \leq n$. Without loss of generality, suppose the responses $\mathbf{y}_1,\mathbf{y}_2, \cdots,\mathbf{y}_n \in \mathcal{Y}$ are ordered such that $r_1 \ge r_2 \ge \cdots \ge r_n$. Otherwise, we can sort $\mathbf{y}_i$ according to $r_i$ in $O(n\log n)$ computational cost. We propose the following reward adjustment model
\begin{equation}
\label{eq: mainproblem}
    \begin{aligned}
        \max \limits_ {\mathbf{z} \in \mathbb{R}^n} & \quad f(\mathbf{z}) := \sum_{i=1}^n p_i z_i^2  \\
        \mathrm{s. t. } & \quad m \le z_i \le M \quad \forall 1 \le i \le n, \\
        & \quad \sum_{i=1}^n p_i z_i = \sum_{i=1}^n p_i r_i, \\
        & \quad z_i \ge z_{i+1} \quad \forall 1 \le i \le n-1,
    \end{aligned}
\end{equation}
where $p_i = \pi_\theta (\mathbf{y}_i | \mathbf{x})>0$. Denote the feasible set of \eqref{eq: mainproblem} as
\[
\mathcal{P} := \left\{\mathbf{z} \in \mathbb{R}^n ~|~ \sum_{i=1}^n p_i z_i = \sum_{i = 1}^n p_i r_i, ~ m \leq z_i \leq M ~ \forall 1 \leq i \leq n, ~ z_{i} \geq z_{i + 1} ~\forall 1 \leq i \leq n - 1 \right\}.
\]
The reward adjustment model \eqref{eq: mainproblem} is intuitive and self-explanatory. Specifically, the objective function measures the reward variance over the responses, and the second and third constraints are proposed to preserve the reward expectation and relative preference over the responses. After obtaining a solution $\mathbf{z}^*$ to the problem \eqref{eq: mainproblem}, we can assign $\tilde{r}(\mathbf{x}, \mathbf{y}_i) = z^*_i$ for $1 \leq i \leq n$.

Realizing that the reward adjustment model \eqref{eq: mainproblem} maximizes a convex function over a convex set, which is known to be NP-hard to solve in general \cite{zwart1974global}. To address the computational challenges for solving \eqref{eq: mainproblem}, we will introduce an efficient algorithm with $O(n)$ computational cost to find a global solution to \eqref{eq: mainproblem} in \cref{subsec:Onalg}, which is based on an explicit characterization of the extreme points of $\mathcal{P}$, shown in \cref{subsec:extremepoint}. Before we introduce the efficient algorithm for solving \eqref{eq: mainproblem}, we first show that the desirable adjusted rewards with higher reward variance over the response space can be obtained by solving \eqref{eq: mainproblem} over a group of generated responses. 

\subsection{A Guarantee for Reward Variance Increase}
\label{subsec:VIguarantee}
The theoretical guarantee for the increase of the reward variance over the response space is established by the following theorem.

\begin{theorem}
\label{thm:varinc}
    Let $\mathbf{x} \in \mathcal{D}$ be a given prompt and $\{ \mathbf{y}_1, \mathbf{y}_2, \cdots, \mathbf{y}_n \} \subset \mathcal{Y}$ be some responses. If the reward $(\tilde{r}(\mathbf{x},\mathbf{y}_1), \tilde{r}(\mathbf{x},\mathbf{y}_2), \cdots, \tilde{r}(\mathbf{x},\mathbf{y}_n) )$ is a global optimal solution of problem \eqref{eq: mainproblem}, with $r_i = r(\mathbf{x}, \mathbf{y}_i)$ and $ p_i = \pi_{\theta}(\mathbf{y}_i | \mathbf{x}) ~ \forall i \in \{1,2,\cdots,n\}$, then the reward variance of policy model $\pi_{\theta}(\cdot | \mathbf{x})$ over the response space can be increased for the prompt $\mathbf{x} \in \mathcal{D}$.
\end{theorem}

\begin{proof}
    For simplicity, denote the response group for prompt $\mathbf{x}$ as $\mathcal{G} := \{ \mathbf{y}_1, \mathbf{y}_2, \cdots, \mathbf{y}_n \}$. The reward adjustment model yields an adjusted reward, transforming $r(\mathbf{x}, \mathbf{y})$ into
    \[
    r' (\mathbf{x}, \mathbf{y}) := \left\{
    \begin{aligned}
        & \tilde{r}(\mathbf{x}, \mathbf{y}) \quad \forall \mathbf{y} \in \mathcal{G}, \\
        & r(\mathbf{x}, \mathbf{y}) \quad \forall \mathbf{y} \in \mathcal{Y} \setminus \mathcal{G}.
    \end{aligned}
    \right.
    \]
    Since the vocabulary $\mathbb{V}$ is finite, the expectation of adjusted reward $r'$ of policy model $\pi_{\theta}$ is 
    \[
    \begin{aligned}
        \mathbb{E}_{\mathbf{y} \sim \pi_\theta(\cdot| \mathbf{x})} [r'(\mathbf{x}, \mathbf{y})] & = \sum_{\mathbf{y} \in \mathcal{G}} \pi_\theta(\mathbf{y} | \mathbf{x}) r'(\mathbf{x}, \mathbf{y}) + \sum_{\mathbf{y} \in \mathcal{Y} \setminus \mathcal{G}} \pi_\theta(\mathbf{y} | \mathbf{x}) r'(\mathbf{x}, \mathbf{y}) \\
        & = \sum_{\mathbf{y} \in \mathcal{G}} \pi_\theta(\mathbf{y} | \mathbf{x}) \tilde{r}(\mathbf{x}, \mathbf{y}) + \sum_{\mathbf{y} \in \mathcal{Y} \setminus \mathcal{G}} \pi_\theta(\mathbf{y} | \mathbf{x}) r(\mathbf{x}, \mathbf{y}) \\
        & = \sum_{\mathbf{y} \in \mathcal{G}} \pi_\theta(\mathbf{y} | \mathbf{x}) r(\mathbf{x}, \mathbf{y}) + \sum_{\mathbf{y} \in \mathcal{Y} \setminus \mathcal{G}} \pi_\theta(\mathbf{y} | \mathbf{x}) r(\mathbf{x}, \mathbf{y}) \\ 
        & = \mathbb{E}_{\mathbf{y} \sim \pi_\theta(\cdot| \mathbf{x})} [r(\mathbf{x}, \mathbf{y})].
    \end{aligned}
    \]
    The third equation holds because the adjusted reward $\tilde{r}(\mathbf{x}, \mathbf{y})$ satisfies the second constraint in \eqref{eq: mainproblem}, which preserves the reward expectation. This implies the reward adjustment model does not change the reward expectation over the response space. Therefore, 

     \[
    \begin{aligned}
        & \mathrm{Var}_{\mathbf{y} \sim \pi_\theta(\cdot|\mathbf{x})} [r'(\mathbf{x}, \mathbf{y})] - \mathrm{Var}_{\mathbf{y} \sim \pi_\theta(\cdot|\mathbf{x})} [r(\mathbf{x}, \mathbf{y})] \\
        = & \mathbb{E}_{\mathbf{y} \sim \pi_\theta(\cdot|\mathbf{x})}[(r'(\mathbf{x}, \mathbf{y}))^2] - \mathbb{E}_{\mathbf{y} \sim \pi_\theta(\cdot|\mathbf{x})}[(r(\mathbf{x}, \mathbf{y}))^2] \\
        = & \sum_{\mathbf{y} \in \mathcal{G}} \pi_\theta(\mathbf{y} | \mathbf{x}) (r'(\mathbf{x}, \mathbf{y}))^2 + \sum_{\mathbf{y} \in \mathcal{Y} \setminus \mathcal{G}} \pi_\theta(\mathbf{y} | \mathbf{x}) (r'(\mathbf{x}, \mathbf{y}))^2 \\
        & - \left(\sum_{\mathbf{y} \in \mathcal{G}} \pi_\theta(\mathbf{y} | \mathbf{x}) (r(\mathbf{x}, \mathbf{y}))^2 + \sum_{\mathbf{y} \in \mathcal{Y} \setminus \mathcal{G}} \pi_\theta(\mathbf{y} | \mathbf{x}) (r(\mathbf{x}, \mathbf{y}))^2 \right) \\
        = & \sum_{\mathbf{y} \in \mathcal{G}} \pi_\theta(\mathbf{y} | \mathbf{x}) (\tilde{r}(\mathbf{x}, \mathbf{y}))^2 - \sum_{\mathbf{y} \in \mathcal{G}} \pi_\theta(\mathbf{y} | \mathbf{x}) (r(\mathbf{x}, \mathbf{y}))^2 \ge 0.
    \end{aligned}
    \]
    The last inequality holds since $(\tilde{r}(\mathbf{x}, \mathbf{y}_1),\cdots, \tilde{r}(\mathbf{x}, \mathbf{y}_n))$ is the global optimal solution of problem \eqref{eq: mainproblem}. That completes the proof.
\end{proof}

\subsection{An Explicit Characterization of Extreme Points of $\mathcal{P}$}
\label{subsec:extremepoint}

Note that the feasible set $\mathcal{P}$ of \eqref{eq: mainproblem} is a convex polyhedral set, which is bounded and nonempty. Therefore, the number of vertices (extreme points) of $\mathcal{P}$ is nonempty and finite, denoted as $\mathcal{V} := \{\mathbf{v}_1, \dots, \mathbf{v}_K\}\subset \mathbb{R}^n$ for some $K > 1$. Moreover, by the representation theorem \cite{grunbaum1967convex}, the polyhedron $\mathcal{P}$ is the convex hull of its extreme point set $\mathcal{V}$. Since the objective function of \eqref{eq: mainproblem} is convex, it is straightforward to show that the maximum of \eqref{eq: mainproblem} is attained at some extreme point of $\mathcal{P}$. For completeness, we include the proof below. 

\begin{lemma}
\label{thm: optimal-extreme}
Define a subset of extreme point set $\mathcal{V}$ as
\[
\mathcal{V}^* := \{\mathbf{v} \in \mathcal{V} ~|~ f(\mathbf{v}) = \max_{\mathbf{v}' \in \mathcal{V}} f(\mathbf{v}')\},
\]
then the point(s) in $\mathcal{V}^*$ are global solution(s) to the nonconvex optimization problem \eqref{eq: mainproblem}.
\end{lemma}
\begin{proof}
For any $\mathbf{z} \in \mathcal{P}$, since $\mathcal{P} = \mbox{conv}(\mathcal{V})$, there exist $0 \leq \lambda_i \leq 1$ satisfying $\lambda_1 + \cdots +\lambda_K =1$ such that 
\[
\mathbf{z} = \lambda_1 \mathbf{v}_1 + \cdots + \lambda_K \mathbf{v}_K. 
\]
Therefore, 
\begin{align*}
f(\mathbf{z}) \leq \sum_{i=1}^K \lambda_i f(\mathbf{v}_i) \leq \max_{1 \leq i \leq K} f(\mathbf{v}_i).
\end{align*}
This completes the proof. 
\end{proof}

The above lemma implies that we can find a global solution to \eqref{eq: mainproblem} by finding the maximum over the set of extreme points $\mathcal{V}$. This relies on an explicit characterization of $\mathcal{V}$, which is shown in the following lemma. 
\begin{lemma}
\label{lemma: charact-extreme-points}
The set of extreme points $\mathcal{V}$ of $\mathcal{P}$ is  
\begin{equation}
\label{eq:extreme_point_set}
\mathcal{V} = \left\{\mathbf{v} = (v_1,\cdots, v_n) \in \mathcal{P} \;\middle|\; 
\begin{aligned}
\exists 0 \leq k \leq l \leq n ~ \mathrm{ and } ~ m < \alpha < M ~ \mathrm{ s.t. }\\
v_1 = \cdots = v_k = M,\\
v_{k+1} = \cdots = v_l = \alpha,\\
v_{l+1} = \cdots = v_n = m,\\
\alpha = (\sum_{i=1}^n p_ir_i - M \sum_{i = 1}^k p_i - m \sum_{i = l+1}^n p_i)/(\sum_{i = k+1}^l p_i).
\end{aligned}
\right\}
\end{equation}
\end{lemma}
\begin{proof}
Define $\mathcal{U} := $ the right hind side of \eqref{eq:extreme_point_set}. We prove $\mathcal{V} = \mathcal{U}$.

Let $\mathbf{v} \in \mathcal{V}$ be any extreme point of $\mathcal{P}$, which is also a basic feasible solution to $\mathcal{P}$. Therefore, there exist $n$ linearly independent active constraints for $\mathbf{v}$. 

From the definition of $\mathcal{P}$, without loss of generality, we assume there exists $0 \leq k \leq l \leq n$, such that 
\[
v_1 = \cdots = v_k = M, \quad v_{l+1} = \cdots = v_n = m. 
\]
\noindent\textbf{Claim}: there exists $m < \alpha < M$ such that 
\begin{equation*}
\label{claim: alpha-part}
v_{k+1} = \cdots = v_l = \alpha.
\end{equation*}

When $l \leq k + 1$, the claim holds automatically. For the case $n \geq l > k+1$, we prove the claim by contradiction. 

Assume that there exist $k+1 \leq j_1 < j_2 \leq l$ and $m < \alpha_2 < \alpha_1 < M$ such that 
\begin{equation*}
\label{eq: contradiction-alpha-part}
M > v_{k+1} = \cdots = v_{j_1} = \alpha_1 > \alpha_2 = v_{j_1 + 1} = \cdots = v_{j_2} > v_{j_2 + 1} \geq m. 
\end{equation*}
Define 
\[
\bar{\delta} := \frac{1}{2}\min\{(M - \alpha_1), (\alpha_1 - \alpha_2), (\alpha_2 - v_{j_2 + 1})\}.
\]
Choose $\delta > 0$ small enough such that 
\[
\delta_1 := \frac{\sum_{i = k+1}^{j_1}p_i}{\sum_{i = j_1+1}^{j_2} p_i}\delta < \bar{\delta}, \quad \delta < \bar{\delta}.
\]
Therefore, we can construct $\widehat{\mathbf{v}},\widetilde{\mathbf{v}}$ as
\[
\widehat{v}_i = \left\{
\begin{array}{ll}
 v_i   &  \mbox{ if } i \leq k \mbox{ or } i > j_2,\\
 \alpha_1 + \delta  & \mbox{ if } k + 1 \leq i \leq j_1,\\
 \alpha_2 - \delta_1  & \mbox{ if } j_1 + 1 \leq i \leq j_2,\\
\end{array}
\right. \quad 
\widetilde{v}_i = \left\{
\begin{array}{ll}
 v_i   &  \mbox{ if } i \leq k \mbox{ or } i > j_2,\\
 \alpha_1 - \delta  & \mbox{ if } k + 1 \leq i \leq j_1,\\
 \alpha_2 + \delta_1  & \mbox{ if } j_1 + 1 \leq i \leq j_2.\\
\end{array}
\right.
\]
It is easy to check that $\widehat{\mathbf{v}},\widetilde{\mathbf{v}} \in \mathcal{P}$. This contradicts the fact that $\mathbf{v}$ is an extreme point of $\mathcal{P}$, since
\[
\widehat{\mathbf{v}} \neq \widetilde{\mathbf{v}} \quad \mbox{and} \quad \mathbf{v} = \frac{1}{2}(\widehat{\mathbf{v}} + \widetilde{\mathbf{v}}).
\]
Thus, the components of $\mathbf{v}$ indexed from $k+1$ to $l$ have one single value $\alpha$ at most. Therefore, we prove the claim. This means $\mathcal{V} \subset \mathcal{U}$.

For each $\mathbf{v} = (v_1,v_2,\cdots,v_n) \in \mathcal{U}$, we prove $\mathbf{v}$ is an extreme point of $\mathcal{P}$ by contradiction. Assume $\mathbf{v}$ is not an extreme point, then there exist $\widetilde{\mathbf{v}} = (\tilde{v}_1, \tilde{v}_2, \cdots, \tilde{v}_n) \in \mathcal{P}, \widehat{\mathbf{v}} = (\hat{v}_1, \hat{v}_2, \cdots, \hat{v}_n) \in \mathcal{P}, \widetilde{\mathbf{v}} \ne \mathbf{v},$ and $ \widehat{\mathbf{v}} \ne \mathbf{v}$ such that 
\begin{equation*}
    \mathbf{v} = \frac{1}{2}(\widetilde{\mathbf{v}} + \widehat{\mathbf{v}}).
\end{equation*}

Since $v_1=v_2 = \cdots = v_k = M,$ we have
\begin{equation*}
    \frac{1}{2} (\tilde{v}_i + \hat{v}_i) = M \quad \forall 1 \le i \le k.
\end{equation*}
It implies that $\tilde{v}_1 = \tilde{v}_2 = \cdots = \tilde{v}_k = M$ and $\hat{v}_1 = \hat{v}_2 = \cdots = \hat{v}_k = M$ as $\tilde{v}_i \le M$ and $\hat{v}_i \le M ~ \forall 1 \le i \le n$. Similarly, we have $\tilde{v}_{l+1} = \tilde{v}_{l+2} = \cdots = \tilde{v}_n = m$ and $\hat{v}_{l+1} = \hat{v}_{l+2} = \cdots = \hat{v}_n = m$.

For each $k+1 \le i \le l$, we have 
\begin{equation}
\label{eq:midpoint}
    \frac{1}{2} (\tilde{v}_i + \hat{v}_i) = \alpha.
\end{equation}
\noindent\textbf{Claim}: $\tilde{v}_{k+1} = \tilde{v}_{k+2} = \cdots = \tilde{v}_l$.

We prove this claim by contradiction. Assume there exists $k+1 \le j \le l-1$ such that $\tilde{v}_j > \tilde{v}_{j+1}$. Then we have 
\begin{equation*}
    \alpha = \frac{1}{2} (\tilde{v}_j + \hat{v}_j) > \frac{1}{2}(\tilde{v}_{j+1} + \hat{v}_{j+1}) = \alpha.
\end{equation*}
This leads to a contradiction. So we have $\tilde{v}_{k+1} = \tilde{v}_{k+2} = \cdots = \tilde{v}_l = \alpha$ as $\widetilde{\mathbf{v}}$ satisfies the equality of \eqref{eq: mainproblem}. Combined with \eqref{eq:midpoint}, we have $\hat{v}_{k+1} = \hat{v}_{k+2} = \cdots = \hat{v}_l = \alpha$. This means $\widetilde{\mathbf{v}} = \widehat{\mathbf{v}} = \mathbf{v}$, which is contradict to the assumption. This implies $\mathbf{v} \in \mathcal{V}$.

Therefore, we have $\mathcal{V} = \mathcal{U}$ and complete the proof of the lemma. 

\end{proof}

As an immediate consequence of \cref{lemma: charact-extreme-points}, we can design an $O(n^2)$ extreme point search algorithm (see Algorithm \ref{alg:extreme-points-search}), which globally solves the nonconvex optimization problem \eqref{eq: mainproblem}. However, an $O(n^2)$ algorithm is slow when the number of responses $n$ is large. In the next subsection, we propose an $O(n)$ algorithm (assuming the rewards and probabilities are sorted) for finding a global solution to \eqref{eq: mainproblem}.

\begin{algorithm}[!h]
\caption{Enumeration search algorithm for \eqref{eq: mainproblem}}
	\label{alg:extreme-points-search}
	\begin{algorithmic}[1]
		\STATE \textbf{Input}: $p = (p_1, p_2, \cdots, p_n)$ with $0 < p_i < 1$. A sorted reward $r = (r_1,r_2,\cdots,r_n)$. An upper bound $M$ and a lower bound $m$.
		\STATE \textbf{Output}: A optimal solution $\mathbf{z}^*$ to \eqref{eq: mainproblem} and the optimal objective value $f^*$.
        \STATE\textbf{Initialization}: $c = \sum_{i = 1}^n p_i r_i$, $\mathbf{z}^* = r$, $f^* = 0$, $k^* = 0$, $l^* = 0$, $\alpha^* = 0$.
        \FOR{$k = 0, 1, \dots, n$}
        \FOR{$l = k, k+1, \dots, n$}
        \STATE Compute $\alpha$ based on \cref{lemma: charact-extreme-points}.
        \STATE Compute current objective value $\bar{f}$ based on $c, M, m, p$.
        \IF{$\bar{f} > f^*$}
        \STATE Set $f^* = \bar{f}$, $k^* = k$, $l^* = l$, $\alpha^* = \alpha$.
        \ENDIF
        \ENDFOR
        \ENDFOR
        \STATE Set
        $z^*_1 = \cdots = z^*_{k^*} = M, z^*_{k^*+1} = \cdots = z^*_{l^*} = \alpha^*, z^*_{l^*+1} = \cdots = z^*_n = m$. 
	   \STATE \textbf{return} $\mathbf{z}^* = (z^*_1,z^*_2, \cdots, z^*_n)$ and $f^*$
	\end{algorithmic}
\end{algorithm}

\subsection{An $O(n)$ Search Algorithm for Solving \eqref{eq: mainproblem}}
\label{subsec:Onalg}

We first prove several key lemmas demonstrating that a one-pass search suffices to obtain the global optimal solution. Without loss of generality, we assume $0 \leq k < l - 1 \leq n$ and define 
\[
S_A = \sum_{i = 1}^k p_i, \quad S_C = \sum_{i = l+1}^n p_i, \quad S_B = 1 - S_A - S_C, \quad \alpha = \frac{c - M S_A - m S_C}{S_B}.
\]

\begin{lemma}
    \label{lemma:alphamonotone}
    Some basic inequalities hold
    \begin{enumerate}
        \item $\frac{c - MS_A - mS_C}{S_B} > \frac{c - M(S_A+p_{k+1}) - mS_C}{S_B - p_{k+1}}$ \quad if $M >\alpha$; 
        \item $\frac{c - MS_A - mS_C}{S_B} < \frac{c - M(S_A+p_{k+1}) - mS_C}{S_B - p_{k+1}}$ \quad if $M <\alpha$;
        \item $\frac{c - MS_A - mS_C}{S_B} < \frac{c - MS_A - m(S_C+p_l)}{S_B - p_l}$ \quad if $m < \alpha$; 
        \item $\frac{c - MS_A - mS_C}{S_B} > \frac{c - MS_A - m(S_C+p_l)}{S_B - p_l}$ \quad if $m > \alpha$.
    \end{enumerate}
\end{lemma}

\begin{proof}
    Since $k < l - 1$, we know that $S_B - p_{k+1} > 0$ and $S_B - p_l > 0$. For the first two inequalities, 
    \[
    \begin{aligned}
        & \frac{c - MS_A - mS_C}{S_B} - \frac{c - M(S_A+p_{k+1}) - mS_C}{S_B - p_{k+1}} \\
        = & \frac{(c - MS_A - mS_C)(S_B - p_{k+1}) - [c - M(S_A+p_{k+1}) - mS_C]S_B}{S_B (S_B - p_{k+1})} \\
        = & \frac{p_{k+1}(MS_B + MS_A + mS_C - c)}{S_B (S_B - p_{k+1})} \\
        = & \frac{p_{k+1}(M - \alpha) }{ S_B - p_{k+1}}.
    \end{aligned}
    \]
    For the last two inequalities.
    \[
    \begin{aligned}
        & \frac{c - MS_A - mS_C}{S_B} - \frac{c - MS_A - m(S_C+p_l)}{S_B - p_l} \\
        = & \frac{(c - MS_A - mS_C)(S_B - p_l) - [c - MS_A - m(S_C+p_l))]S_B}{S_B(S_B - p_l)} \\
        = & \frac{p_l (MS_A + mS_B + mS_C - c)}{S_B(S_B - p_l)} \\
        = & \frac{p_l (m - \alpha)}{S_B - p_l}.
    \end{aligned}
    \]
    This completes the proof.
\end{proof}

\begin{lemma}
\label{lemma:infeasible}
Assume $m < \alpha < M$, we have
\begin{enumerate}
    \item if $\frac{c - M(S_A + p_{k+1}) - mS_C}{S_B - p_{k+1}} < m$, then $v_{k+1} < M$ for all $\mathbf{v} \in \mathcal{V}$;
    \item if $\frac{c - MS_A - m(S_C + p_l)}{S_B - p_{l}} > M$, then $v_{l} > m$ for all $\mathbf{v} \in \mathcal{V}$.
\end{enumerate}
\end{lemma}
\begin{proof}
Assume that 
\[
\frac{c - M(S_A + p_{k+1}) - mS_C}{S_B - p_{k+1}} < m.
\]
This is, 
\[
c < M(S_A + p_{k+1}) + mS_C + m(S_B - p_{k+1}).
\]
\noindent\textbf{Claim}: This implies that $v_{k+1} < M$ for all $\mathbf{v} \in \mathcal{V}$.

We prove this claim by contradiction. If $v_{k+1} = M$ for all $\mathbf{v} \in \mathcal{V}$. Then by the equality constraint,
\[
    c \ge M(S_A + p_{k+1}) + m (S_B - p_{k+1} + S_C) > c.
\]
This leads to a contradiction. Thus, the first part of the lemma holds.

Now, assume that 
\[
\frac{c - MS_A - m(S_C + p_l)}{S_B - p_{l}} > M.
\]
This is, 
\[
c > MS_A + M(S_B - p_l) + m(S_C + p_l). 
\]
Similarly, this implies that $v_{l} > m$ for all $\mathbf{v} \in \mathcal{V}$. We have completed the proof. 
\end{proof}

\begin{lemma}
\label{lemma:objmonotone}
Assume $m < \alpha < M$, by the monotonicity proved in \cref{lemma:alphamonotone}, we know that 
\begin{enumerate}
    \item $\frac{c - M(S_A + p_{k+1}) - mS_C}{S_B - p_{k+1}} < M$;
    \item $\frac{c - MS_A - m(S_C + p_l)}{S_B - p_{l}} > m$.
\end{enumerate}
Moreover, 
\begin{itemize}
    \item[(a)] if 
    \[
    \frac{c - M(S_A + p_{k+1}) - mS_C}{S_B - p_{k+1}} > m,
    \]
    then it can strictly increase the objective function value by setting $k = k+1$ and $l = l$;
    \item[(b)] if 
    \[
    \frac{c - MS_A - m(S_C + p_l)}{S_B - p_{l}} <M,
    \]
     then it can strictly increase the objective function value by setting $k = k$ and $l = l - 1$.
\end{itemize}
\end{lemma}

\begin{proof}
We prove the strict monotonicity of the objective function value. For convenience, denote 
\[
D = c - MS_A - mS_C. 
\]
Note that 
\[
f_{new} = (S_A + p_{k+1})M^2 + S_Cm^2 + \frac{(D - Mp_{k+1})^2}{S_B - p_{k+1}}, \quad f_{old} = S_AM^2 + S_Cm^2 + \frac{D^2}{S_B}.
\]
\begin{eqnarray*}
f_{new} - f_{old} &=& p_{k+1}M^2 + \frac{(D - Mp_{k+1})^2}{S_B - p_{k+1}} - \frac{D^2}{S_B}\\[5pt]
&=& \frac{p_{k+1} M^2 S_B(S_B - p_{k+1}) + S_B(D - Mp_{k+1})^2 - D^2(S_B - p_{k+1})}{S_B(S_B - p_{k+1})}\\[5pt]
&=& \frac{p_{k+1} M^2 S_B^2 - p_{k+1}^2 M^2 S_B + S_B D^2 - 2S_B D Mp_{k+1} + S_B M^2 p_{k+1}^2 - D^2S_B + D^2p_{k+1}}{S_B(S_B - p_{k+1})}\\[5pt]
&=& \frac{p_{k+1} M^2 S_B^2 - 2 S_B D M p_{k+1}  + D^2p_{k+1}}{S_B(S_B - p_{k+1})}\\[5pt]
&=& \frac{p_{k+1}(MS_B - D)^2}{S_B(S_B - p_{k+1})}\\
&>& 0.
\end{eqnarray*}
This proves part (a). Part (b) can be proved similarly.
\end{proof}

Based on \cref{lemma:alphamonotone}, \cref{lemma:infeasible} and \cref{lemma:objmonotone}, an $O(n)$ algorithm for optimal solution search can be designed. The proposed algorithm employs a bidirectional search strategy, initiating from both ends of an initial solution and progressing toward the center. During the search process, the objective function value exhibits a monotonically increasing trend according to \cref{lemma:objmonotone}. Furthermore, \cref{lemma:alphamonotone} guarantees that once a solution becomes infeasible, all subsequent solutions along that search direction are infeasible, thereby providing a natural termination criterion for this algorithm along this search direction. The complete procedure is presented in detail in \cref{alg:one_pass}.

\begin{algorithm}[!h]
\caption{One-pass search algorithm for \eqref{eq: mainproblem}}
	\label{alg:one_pass}
	\begin{algorithmic}[1]
		\STATE \textbf{Input}: $p = (p_1, p_2, \cdots, p_n)$ with $0 < p_i < 1$. A sorted reward $r = (r_1,r_2,\cdots,r_n)$. An upper bound $M$ and a lower bound $m$.
		\STATE \textbf{Output}: A optimal solution $\mathbf{z}^*$ to \eqref{eq: mainproblem} and the optimal objective value $f^*$.
        \STATE\textbf{Initialization}: $c = \sum_{i = 1}^n p_i r_i$, $k = 0$, $l = n$, $iter=0$.
        \STATE Define $C(0) = 0$ and compute the cumulative sums 
        $C(i) = \sum_{j = 1}^i p_j \quad 1\leq i \leq n$.
        \WHILE{$iter < n$}
        \STATE Compute
        $S_A = C(k), S_C = C(n) - C(l), S_B = C(l) - C(k)$.
        \IF{$S_B \le 0$}
        \STATE \textbf{Break.}
        \ENDIF
        \STATE Set $\alpha = (c-MS_A-mS_C)/S_B$. Compute objective $f = S_AM^2 + S_B\alpha^2 + S_Cm^2$.
        \STATE Set $f^*=f, k^*=k, l^*=l, flags=0$.
        \IF{$k<l$ and $S_B-p_{k+1}>0$}
        \STATE $\bar{\alpha} = \frac{c-(S_A+p_{k+1})M-S_cm}{S_B-p_{k+1}}$.
        \IF{$m \le \bar{\alpha} \le M$}
        \STATE $f_{new} = (S_A+p_{k+1})M^2 + (S_B-p_{k+1}) \bar{\alpha}^2 +S_C m^2$.
        \IF{$f_{new} > f^*$}
        \STATE $f^* = f_{new}, k^* = k+1, l^* = l, flags=1$.
        \ENDIF
        \ENDIF
        \ENDIF
        \IF{$l>k$ and $S_B-p_l>0$}
        \STATE $\bar{\alpha} = \frac{c-S_AM-(S_C+p_l)m}{S_B-p_l}$.
        \IF{$m \le \bar{\alpha} \le M$}
        \STATE $f_{new} = S_AM^2 + (S_B-p_l) \bar{\alpha}^2 + (S_C+p_l)m^2$.
        \IF{$f_{new} > f^*$}
        \STATE $f^* = f_{new}, k^* = k, l^* = l-1, flags=1$.
        \ENDIF
        \ENDIF
        \ENDIF
        \IF{$flags=1$}
        \STATE $k = k^*, l = l^*$.
        \STATE \textbf{Continue.}
        \ELSE
        \STATE \textbf{Break.}
        \ENDIF
        \STATE $iter = iter + 1 $.
	\ENDWHILE	
        \STATE Set $z_1^* = \cdots = z_{k^*}^* = M, z_{k^*+1}^* = \cdots = z_{l^*}^* = \alpha, z_{l^*+1}^* = \cdots = z_n^* = m$.
	\STATE \textbf{return} $\mathbf{z}^*=(z_1^*,z_2^*,\cdots,z_n^*)$ and $f^*$
	\end{algorithmic}
\end{algorithm}

\section{A GRPOVI Algorithm for Faster RLHF Training}
\label{sec:GRPOalg}

In this section, we present the GRPO with reward variance increase (GRPOVI) algorithm for faster RLHF training. We first introduce how to integrate the reward adjustment model into the GRPO training. For simplicity, let the policy model at time step $t$ be $\pi_{\theta(t)}$. For each prompt, the responses are generated from the current policy model $\pi_{\theta(t)}(\cdot | \mathbf{x})$. It follows from \cite[Theorem 4]{razin2025makes} that the efficiency of the RLHF training depends on the reward variance at time step $t = 0$. Thus, we apply \eqref{eq: mainproblem} to compute the adjusted reward with $p_i = \pi_{\theta(0)}(\mathbf{y}_i | \mathbf{x})$ using the initial policy model $\pi_{\theta(0)}(\cdot | \mathbf{x})$. The adjusted rewards are used to compute the advantages in GRPO. The following corollary guarantees the reward variance increases over the response space, which is a direct consequence of \cref{thm:varinc}. 

\begin{corollary}
\label{col:varinc}
    Let $\mathbf{x} \in \mathcal{D}$ be any given prompt. Assume that the responses $\{ \mathbf{y}_1, \mathbf{y}_2, \cdots, \mathbf{y}_n \}$ are generated from the policy model $\pi_{\theta(t)}(\cdot|\mathbf{x})$. If the reward $(\tilde{r}(\mathbf{x},\mathbf{y}_1), \tilde{r}(\mathbf{x},\mathbf{y}_2), \cdots, \tilde{r}(\mathbf{x},\mathbf{y}_n) )$ is a global optimal solution of problem \eqref{eq: mainproblem}, with $r_i = r(\mathbf{x}, \mathbf{y}_i)$ and $p_i = \pi_{\theta(0)}(\mathbf{y}_i | \mathbf{x}) ~ \forall i \in \{1,2,\cdots,n\}$, then the reward variance of the initial policy model $\pi_{\theta(0)}(\cdot | \mathbf{x})$ over the response space can be increased for the prompt $\mathbf{x} \in \mathcal{D}$.
\end{corollary}

The GRPOVI algorithm applies GRPO for RLHF training with adjusted rewards for the generated responses via solving \eqref{eq: mainproblem}. According to \cref{thm:varinc} and \cite[Theorem 4]{razin2025makes}, the GRPOVI algorithm achieves faster RLHF training compared to the original GRPO algorithm. The details of the GRPOVI algorithm are presented in \cref{alg:grpovi}.

\begin{algorithm}[!h]
\caption{GRPOVI: GRPO with reward variance increase}
	\label{alg:grpovi}
	\begin{algorithmic}[1]
		\STATE \textbf{Input}: reference model $\pi_{\mathrm{ref}}$, reward model $r$, dataset $\mathcal{D}$, group size $n$.
		\STATE \textbf{Output}: policy model $\pi_{\theta(T)}$.
        \STATE\textbf{Initialization}: $\pi_{\theta(0)} = \pi_{\mathrm{ref}}$.
	\FOR{step $t=0,1,\cdots, T$}
        \STATE Choose a batch $\mathcal{D}_b$ from $\mathcal{D}$.
        \STATE Generate $n$ responses $\{\mathbf{y}_i\}_{i=1}^n$ from $\pi_{\theta(t)}(\cdot|\mathbf{x})$ for each $\mathbf{x} \in \mathcal{D}_b$.
        \STATE Compute the probability $p_i = \pi_{\mathrm{ref}}( \mathbf{y}_i | \mathbf{x})$ based on reference model $\pi_{\mathrm{ref}}$ for each $i=1,2,\cdots,n$.
        \STATE Normalize the probabilities $p_i = \frac{p_i}{\sum_{i=1}^n p_i}$.
        \STATE Compute reward $r_i = r(\mathbf{x},\mathbf{y}_i)$ for each $i=1,2,\cdots,n$.
        \STATE Sort $\{p_i\}_{i=1}^n$ and $\{r_i\}_{i=1}^n$ such that $r_1\ge r_2 \ge \cdots \ge r_n$.
        \STATE Compute the adjusted rewards $\{\tilde{r}_i\}_{i=1}^n$ based on $\{p_i\}_{i=1}^n$ and $\{r_i\}_{i=1}^n$ (\cref{alg:extreme-points-search} or \cref{alg:one_pass}).
        \STATE Compute advantages in GRPO based on $\{\tilde{r}_i\}_{i=1}^n$.
        \STATE Update policy model $\pi_{\theta(t)}$ according to GRPO-based RLHF training method.
        \ENDFOR
	\STATE \textbf{return} policy model $\pi_{\theta(T)}$
	\end{algorithmic}
\end{algorithm}

Notably, benefiting from the fast \cref{alg:one_pass} for solving \eqref{eq: mainproblem}, the additional per-iteration training cost of the GRPOVI algorithm for RLHF training is marginal compared to the original GRPO algorithm. Furthermore, as demonstrated in \cref{sec:experiments}, GRPOVI achieves much higher rewards than the standard GRPO approach after the same number of training steps. Moreover, an important observation emerges from our reward adjustment process: while the original rewards $\{ r(\mathbf{x}, \mathbf{y}_i) \}_{i=1}^n$ may assume arbitrary values, the adjusted rewards $\{ \tilde{r} (\mathbf{x}, \mathbf{y}_i) \}_{i=1}^n$ exhibit at most three distinct values. This ternary reward structure effectively simulates a rule-based reward mechanism that categorizes responses into ``positive'', ``neutral'', or ``negative'' evaluations. This property indirectly explains the empirical effectiveness of rule-based reward in GRPO applications like DeepSeek-R1 \cite{shao2024deepseekmath}, as such discrete reward schemes inherently promote higher variance in the reward distribution.

\section{Experiments}
\label{sec:experiments}
This section encompasses two primary experiments.
First, we evaluate the effectiveness and computational efficiency of our proposed $O(n)$ algorithm (\cref{alg:one_pass}) for solving \eqref{eq: mainproblem} through a sequence of simulation studies. Subsequently, we conduct extensive experiments to demonstrate the superior performance of the proposed GRPOVI algorithm compared to the standard GRPO algorithm.

\subsection{Experiments on the efficiency of  \cref{alg:one_pass}}
\label{subsec:search_alg_exp}

To showcase the performance of the One-pass search algorithm for solving the optimization problem \eqref{eq: mainproblem}. We first randomly generate the probabilities $\{p_i\}_{i=1}^n$ and rewards $\{ r_i \}_{i=1}^n$ with different sizes. 
Then these sorted vectors are input for testing. The results are presented in \cref{tab:optalg}. The numerical experiments in \cref{tab:optalg} are run on an INTEL(R) XEON(R) GOLD 6548Y+ CPU.

\begin{table}[h]
\centering
\begin{tabular}{c|cccc}
\hline
\multirow{2}{*}{Size $n$} & \multicolumn{2}{c}{Enumeration search algorithm} & \multicolumn{2}{c}{One-pass search algorithm}  \\
\cline{2-5}
& Optimal value $f_1^*$ & Running time (s) & Optimal value $f_2^*$ & Running time (s) \\
\hline
10 & 0.3918 & 0.0022 & 0.3918 & 0.0015  \\
50 & 0.4260 & 0.0438 & 0.4260 & 0.0057  \\
100 & 0.5159 & 0.1707 & 0.5159 & 0.0112  \\
500 & 0.4970 & 4.2043 & 0.4970 & 0.0539  \\
1000 & 0.4938 & 16.9763 & 0.4938 & 0.1106  \\
5000 & 0.4919 & 408.9071 & 0.4919 & 0.5515  \\
10000 & 0.4933 & 1659.4376 & 0.4933 & 1.0561  \\
\hline
\end{tabular}
\caption{Comparison of enumeration search algorithm (\cref{alg:extreme-points-search}) and one-pass search algorithm (\cref{alg:one_pass}) for solving \eqref{eq: mainproblem}.}
\label{tab:optalg}
\end{table}

It is worth noting that the optimal value $f_1^*$ of \cref{alg:extreme-points-search} corresponds to the global solution since it enumerates all the extreme points. As shown in \cref{tab:optalg}, the One-pass search algorithm also achieves the global optimum. Furthermore, when the size of group $n$ is relatively small, the running times of both algorithms are nearly identical. Therefore, either algorithm can be integrated into the GRPOVI framework when the number of responses is small. However, as the group size increases, the One-pass algorithm exhibits a significant advantage in computational efficiency.

\subsection{Experiments on the Superior Performance of the GRPOVI Algorithm for RLHF Training}
In this subsection, we compare the performance of the proposed GRPOVI algorithm and the original GRPO algorithm for preference alignment of the LLMs. All the experiments in this subsection are run on an NVIDIA A100-SXM4-80GB GPU with CUDA 12.4. 

\subsubsection{Experiment Settings}

We summarize the experiment settings as follows.

\noindent \textbf{Reward models.} Following \cite{razin2025makes}, the reward models used in this paper for preference alignment are GRM-Gemma-2-2B \cite{yang2024regularizing} and GRM-Llama-3.2-3B \cite{yang2024regularizing}, both of which are fully pre-trained. In the original GRPO algorithm, the initial policy model $\pi_{\mathrm{ref}}$ is trained based on the reward model $r$. In contrast, the GRPOVI algorithm adjusts the rewards over each group to $\tilde{r}$ by solving \eqref{eq: mainproblem}. Inspired by the prior work \cite{tang2024understanding, razin2025makes, chen2024accuracy}, we use a high-quality reward model to evaluate the goodness of the trained policy model. Specifically, we adopt the ArmoRM reward model \cite{wang2024interpretable} as the ground truth since it achieves high scores on RewardBench \cite{lambert2024rewardbench}.

\noindent \textbf{Dataset.} We utilize the standard RLHF dataset, UltraFeedback \cite{cui2024ultrafeedback}, for our experiments. Similar to \cite{razin2025makes}, the dataset is split into three parts: the policy gradient training part, the test part and the remaining part. The remaining part is used for SFT the language model, with the prompts sampled from it and the preferred outputs selected by the ground truth reward model. The first two parts are applied for policy gradient training and testing after SFT phase.

\noindent \textbf{Initial policy model.} We use Pythia \cite{biderman2023pythia} as the initial policy model, fine-tuned via SFT on UltraFeedback. We then perform RLHF training using both the GRPOVI algorithm and the original GRPO algorithm, comparing their performance after an equivalent number of training steps.

\noindent \textbf{Reward normalization.} Similar to \cite{gao2023scaling, razin2025makes}, the reward models are normalized so that they will produce rewards on the same scale. Specifically, 500 prompts are sampled from the policy gradient training set, and 10 outputs are generated per prompt from the initial policy model. For each reward model (including the ground truth reward model), the mean and standard deviation are calculated across these 5000 outputs, thereby normalizing the rewards in subsequent experiments.

\noindent \textbf{Policy gradient.} For GRPO-based RLHF training, the group size is set to 8, meaning 8 responses are generated per prompt. The batch size is set to 32, corresponding to 4 prompts per batch from the training set. The learning rate is set to $1\times 10^{-7}$. The policy gradient method is conducted by passing over all the prompts from the training set once, and 8 checkpoints are recorded evenly during the training process. 

\subsubsection{Results} 

We employ the GRPOVI algorithm and the original GRPO algorithm to do preference alignments for the initial policy model with the two aforementioned reward models. To comprehensively compare the performance of these two algorithms throughout the training process, we save eight checkpoints for each algorithm. We evaluate each checkpoint by computing the average ground truth reward on both training and test sets, as depicted in \cref{fig:gemma_all} and \cref{fig:llama_all}, respectively. To mitigate the impact of randomness, we repeated the RLHF training process 4 times for each algorithm configuration.  

\begin{figure}[!ht]
    \centering
    \subfloat[Performance on the training set.]{\includegraphics[width=0.65\textwidth]{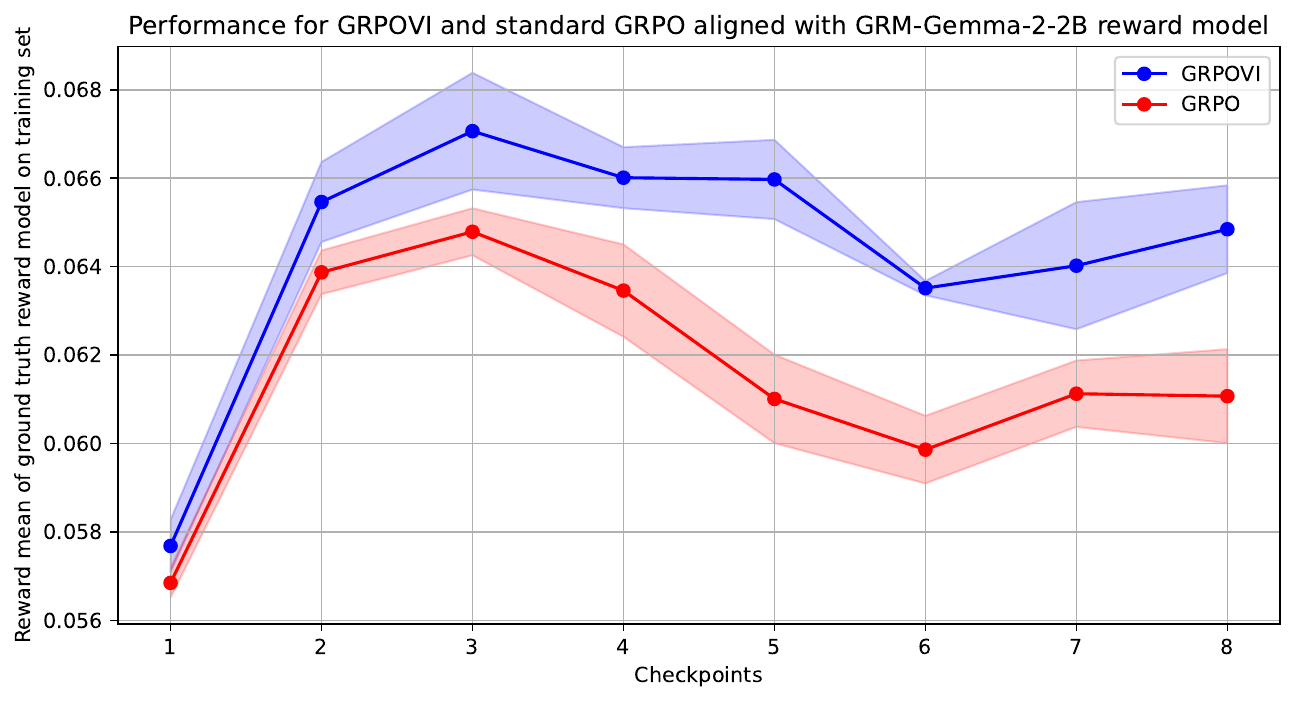}\label{fig:gemma_train}} \\
    \subfloat[Performance on the test set.]{\includegraphics[width=0.65\textwidth]{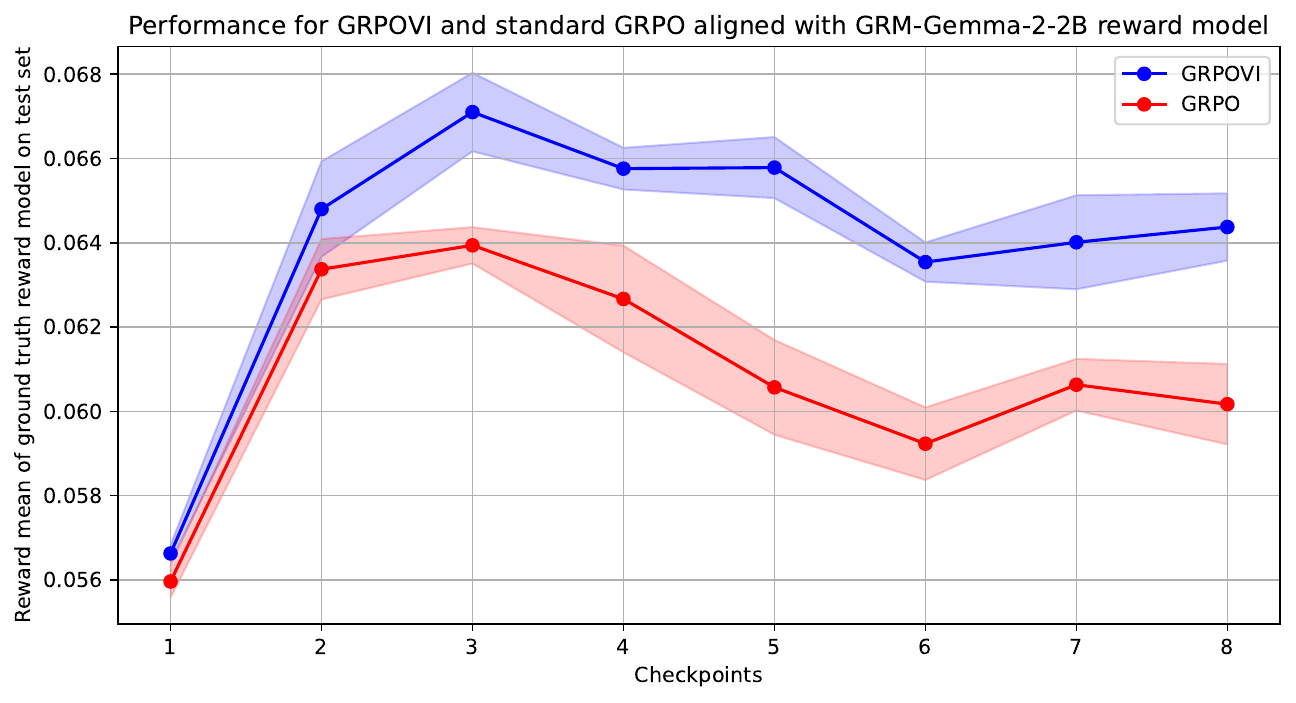}\label{fig:gemma_test}}
    \caption{Comparison of GRPOVI and original GRPO algorithms aligned with GRM-Gemma-2-2B reward model. Each dot represents the average of 4 independent training results at the same checkpoint. The shaded areas are constructed by adding and subtracting the standard deviation from the average.}
    \label{fig:gemma_all}
\end{figure}

\begin{figure}[!ht]
    \centering
    \subfloat[Performance on the training set.]{\includegraphics[width=0.65\textwidth]{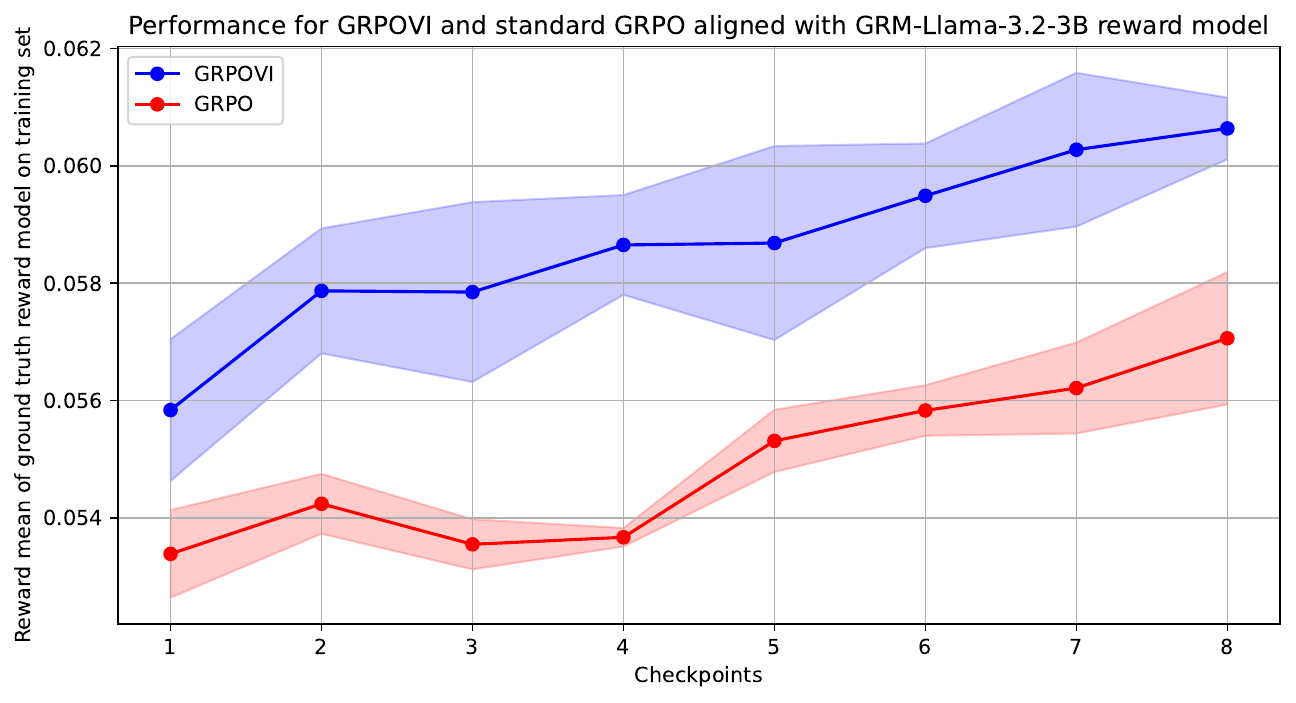}\label{fig:llama_train}} \\
    \subfloat[Performance on the test set.]{\includegraphics[width=0.65\textwidth]{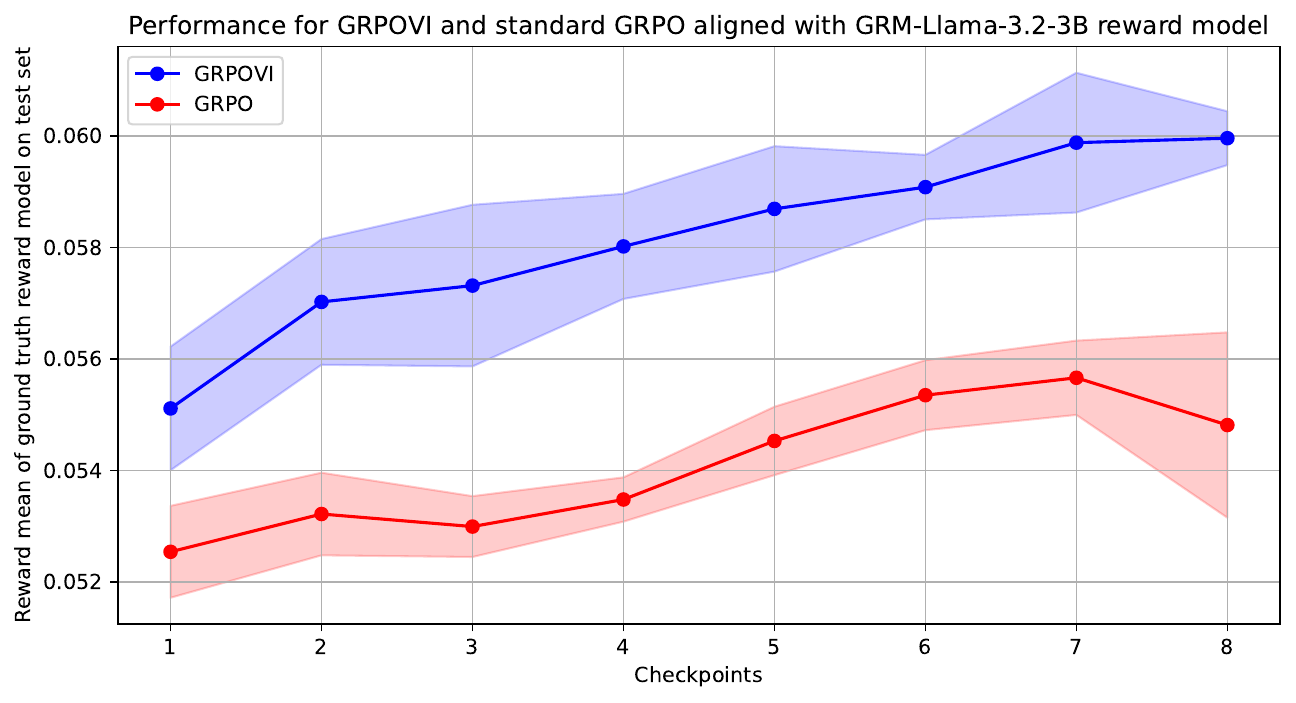}\label{fig:llama_test}}
    \caption{Comparison of GRPOVI and original GRPO algorithms aligned with GRM-Llama-3.2-3B reward model. Each dot represents the average of 4 independent training results at the same checkpoint. The shaded areas are constructed by adding and subtracting the standard deviation from the average.}
    \label{fig:llama_all}
\end{figure}

The experiment results demonstrate that the GRPOVI algorithm significantly outperforms the original GRPO after the same number of training steps. These findings confirm that increasing reward variance before policy gradient optimization can effectively accelerate the GRPO-based RLHF training. Moreover, compared to the original GRPO algorithm, GRPOVI merely integrates an additional $O(n\log n)$ search algorithm. The search algorithm has negligible computational overhead since $n$ is relatively small. Consequently, both GRPO and GRPOVI algorithms require similar per iteration training time.

\section{Conclusions}
\label{sec:conclusions}
In this paper, we propose a reward adjustment model, which leads to a novel GRPOVI algorithm for faster RLHF training. To address the computational challenges in solving the nonconvex reward adjustment model, we explicitly characterize the extreme points of the feasible set of the model. Based on this characterization, we designed a highly efficient $O(n\log n)$ algorithm for finding a global solution to the nonconvex reward adjustment model. As an interesting byproduct, we provide an indirect explanation for the empirical effectiveness of GRPO with rule-based reward for RLHF training, as demonstrated in DeepSeek-R1. The numerical results shown in this paper demonstrate the superior performance of the GRPOVI algorithm compared to the original GRPO algorithm in RLHF training. Our results also provide practical evidence to support the interesting theoretical findings in \cite{razin2025makes}. As a possible future research direction, we are interested in investigating whether the GRPOVI algorithm proposed in this paper contributes to a more robust RLHF training regarding the noise and distribution shifts of the training data.


\bibliography{GRPOVI_arxiv}{}
\bibliographystyle{siam}
\end{document}